\documentclass{article}

\usepackage{microtype}
\usepackage{graphicx}
\usepackage{subfigure}
\usepackage{booktabs} 
\usepackage{amsthm}
\usepackage{bbm}
\usepackage{enumitem}
\usepackage{booktabs} 
\usepackage{amsmath}
\usepackage{bm}
\usepackage{amsfonts}       
\usepackage{nicefrac}       
\usepackage{microtype}      
\usepackage{color}
\usepackage{wrapfig}
\usepackage{makecell}
\usepackage{array}
\usepackage{adjustbox}
\usepackage{url}
\newcolumntype{?}{!{\vrule width 2pt}}
\usepackage{comment}
\usepackage{caption}
\usepackage[T1]{fontenc}
\usepackage{mathtools}
\usepackage[para]{footmisc}
\usepackage{multirow}
\usepackage{stfloats}
\usepackage{footnote}
\usepackage[linesnumbered,algoruled,boxed,lined]{algorithm2e}

\usepackage{hyperref}

\usepackage[accepted]{icml2020}

\icmltitlerunning{On $\ell_p$-norm Robustness of Ensemble Decision Stumps and Trees}

\usepackage{amsmath}
\usepackage{mathrsfs}
\usepackage{color}
\usepackage{multirow}
\usepackage{threeparttable}
\usepackage{graphicx}

\SetKwProg{Fn}{Function}{}{end}
\SetKwFunction{CliqueEnum}{CliqueEnumerate}
\SetKwFunction{MultiVeri}{MultiLevelVerify}
\SetKwFunction{FRecurs}{ComputeRecursive}%
\SetKwFunction{Go}{ComputeBox}
\SetKwFunction{CheckClique}{CheckClique}
\newtheorem{theorem}{Theorem}

\newtheorem{definition}{Definition}
\newtheorem{proposition}{Proposition}

\allowdisplaybreaks

\usepackage{amsmath,amsfonts,bm}

\def\ceil#1{\lceil #1 \rceil}

\def\1{\bm{1}}

\def\rmI{{\mathbf{I}}}

\def\vx{{\bm{x}}}

\def\vdelta{{\bm{\delta}}}

\DeclareMathAlphabet{\mathsfit}{\encodingdefault}{\sfdefault}{m}{sl}
\SetMathAlphabet{\mathsfit}{bold}{\encodingdefault}{\sfdefault}{bx}{n}

\newcommand{\R}{\mathbb{R}}

\DeclareMathOperator*{\argmin}{arg\,min}

\begin{document}

\twocolumn[
\icmltitle{On $\ell_p$-norm Robustness of Ensemble Decision Stumps and Trees}



\icmlsetsymbol{equal}{*}

\begin{icmlauthorlist}
\icmlauthor{Yihan Wang}{tsinghua}
\icmlauthor{Huan Zhang}{ucla}
\icmlauthor{Hongge Chen}{mit}
\icmlauthor{Duane Boning}{mit}
\icmlauthor{Cho-Jui Hsieh}{ucla}
\end{icmlauthorlist}

\icmlaffiliation{tsinghua}{ Tsinghua University, Beijing, China}
\icmlaffiliation{ucla}{UCLA, Los Angeles, USA}
\icmlaffiliation{mit}{MIT, Cambridge, USA}

\icmlcorrespondingauthor{Yihan Wang}{wangyihan617@gmail.com}

\icmlkeywords{Machine Learning, ICML}

\vskip 0.3in
]



\printAffiliationsAndNotice 

\begin{abstract}
Recent papers have demonstrated that ensemble stumps and trees could be vulnerable to small input perturbations, so robustness verification and defense for those models have become an important research problem. However, due to the structure of decision trees, where each node makes decision purely based on one feature value, all the previous works only consider the $\ell_\infty$ norm perturbation. To study robustness with respect to a general $\ell_p$ norm perturbation, one has to consider the correlation between perturbations on different features, which has not been handled by previous algorithms. In this paper, we study the problem of robustness verification and certified defense with respect to general $\ell_p$ norm perturbations for ensemble decision stumps and trees. For robustness verification of ensemble stumps, we prove that complete verification is NP-complete for $p\in(0, \infty)$ while polynomial time algorithms exist for $p=0$ or $\infty$. For $p\in(0, \infty)$ we develop an efficient dynamic programming based algorithm for sound verification of ensemble stumps. For ensemble trees, we generalize the previous multi-level robustness verification algorithm to $\ell_p$ norm.
We demonstrate the first certified defense method for training ensemble stumps and trees with respect to $\ell_p$ norm perturbations, and verify its effectiveness empirically on real datasets.
\end{abstract}

\section{Introduction}

It has been observed that small human-imperceptible perturbations can mislead a well-trained deep neural network~\citep{goodfellow2014explaining,szegedy2013intriguing}, which leads to extensive studies on robustness of deep neural network models. In addition to strong attack methods that can find adversarial perturbations in both white-box~\citep{carlini2017towards,madry2017towards,chen2018attacking,zhang2019limitations,xu2019topology} and black-box settings~\citep{chen2017zoo,ilyas2017query,brendel2017decision,cheng2019query,cheng2020signopt}, various algorithms have been proposed for formal robustness verification~\citep{katz2017reluplex,gehr2018ai2,zhang2018efficient,weng2018towards,zhang2018recurjac,wang2018efficient} and improving the robustness of neural networks~\citep{madry2017towards,kolter2017provable,wong2018scaling,zhang2019theoretically,zhang2019towards}.

In this paper, we consider the robustness of ensemble decision trees and stumps. Although tree based model ensembles, including Gradient Boosting Trees (GBDT)~\citep{friedman2001greedy} and random forest, have been widely used in practice, their robustness properties have not been fully understood. Recently, \citet{cheng2019query,chen2019tree,kantchelian2016evasion} showed that adversarial examples also exist in ensemble trees, and several recent works considered the problem of robustness verification~\cite{chen2019robustness,ranzato2019robustness,ranzato2020abstract,tornblom2019abstraction} and adversarial defense~\cite{chen2019tree,andriushchenko2019provably,wang2019tree,calzavara2019treant,calzavara2020feature,chen2019training} for ensemble trees and stumps. However, most of these works focus on evaluating and enhancing the robustness for $\ell_\infty$ norm perturbations, while $\ell_p$ norm perturbations with $p<\infty$ were not considered. Since each node or each stump makes decision by looking at only a single feature, 
the perturbations are independent across features in $\ell_\infty$ robustness verification and defense for tree ensembles, which makes the problem intrinsically simpler than the other $\ell_p$ norm cases with $p<\infty$. 
In fact, we will show that in some cases verifying $\ell_p$ norm and $\ell_\infty$ norm belong to different complexity classes -- verifying $\ell_p$ norm robustness of an ensemble decision stump is NP-complete for $p\in (0, \infty)$ while polynomial time algorithms exist for $p=0, \infty$. 

In practice, robustness on a single $\ell_\infty$ norm is not sufficient -- it has been demonstrated that an $\ell_\infty$ robust model can still be vulnerable to invisible adversarial perturbations in other $\ell_p$ norms~\citep{schott2018towards,tramer2019adversarial}. Additionally, there are cases where an $\ell_p$ norm threat model is more suitable than $\ell_\infty$ norm. For instance, when the perturbation can be made only to few features, it should be modeled as an $\ell_0$ norm perturbation. Thus, it is crucial to have robustness verification and defense algorithms that can work for general $\ell_p$ norms. In this paper, We give a comprehensive study of this problem for tree based models.
Our contribution can be summarized as follows: 
\begin{itemize}[wide]
    \item  In the first part of paper, we consider the problem of verifying $\ell_p$ norm robustness of tree and stump ensembles. 
    For a  single decision tree, similar to the $\ell_\infty$ norm case, we show that 
    the problem of complete robustness verification of $\ell_p$ norm robustness can be done in linear time.  However, for ensemble decision stump, although complete $\ell_\infty$ norm verification can be done in polynomial time, it's NP-complete for verifying $\ell_p$ norm robustness when $p\in(0, \infty)$. We then provide an efficient algorithm to conduct sound but incomplete verification  by dynamic programming. For tree ensembles, the $\ell_p$ case is  NP-complete for any $p$ and we propose an efficient algorithm for computing a reasonably tight lower bound.
    Table~\ref{tab:summary} the algorithms proposed in our paper and previous works, as well as their complexity. 
    \item 
    Based on the proposed robustness verification algorithms, we develop training algorithms for ensemble stumps and trees that can improve certified robust test errors with respect to general $\ell_p$ norm perturbations. Experiments on multiple datasets verify that the proposed methods can improve $\ell_p$ norm robustness where the previous $\ell_\infty$ norm certified defense~\citep{andriushchenko2019provably} cannot. 
\end{itemize}
The rest of the paper is organized as follows. In Section \ref{sec:background}, we introduce the robustness verification and certified defense problems. In Section \ref{sec:verification}, we discuss complexity and algorithms for $\ell_p$ norm robustness verification for ensemble stumps and trees. In Section \ref{sec:defense}, we show how to use our proposed verification algorithms to train ensemble stumps and trees with certified $\ell_p$ norm robustness. Experiments on multiple datasets are conducted in Section \ref{sec:exp}.

\begin{table*}
\caption{Summary of the algorithms and their complexity for robustness verification of ensemble trees and stumps. Blue cells are the contribution of this paper. }
\resizebox{1\textwidth}{!}{
\begin{tabular}{|c|c|c|c|c|}
\hline
& Verification method & $\ell_\infty$ & $\ell_0$ & $\ell_p, p\in (0, \infty)$ \\
\hline
Single Tree & complete &  Linear \cite{chen2019robustness} & {\color{blue}Linear (Sec 3.1)} & {\color{blue}Linear (Sec 3.1) }\\
\hline
\multirow{2}{*}{Ensemble Stump} & complete & Polynomial~\cite{andriushchenko2019provably} & {\color{blue}Linearithmic (Sec 3.2)} &
{\color{blue}NP-complete (Sec 3.2)}\\
\cline{2-5}
& incomplete & Not needed & {\color{blue}Not needed} & {\color{blue}Approximate Knapsack (Sec 3.2) } \\
\hline
\multirow{2}{*}{Ensemble Tree} & complete & \multicolumn{3}{c|}{NP-complete~\cite{kantchelian2016evasion}} \\
\cline{2-5}
& incomplete &Multi-level~\cite{chen2019robustness} & \multicolumn{2}{c|}{{\color{blue}Extended Multi-level (Sec 3.3)}} \\
\hline
\end{tabular}
}
\label{tab:summary}
\end{table*}

\section{Background and Related Work}
\label{sec:background}
\paragraph{Background} Assume $F: \R^d \rightarrow \{1, \dots, C\}$ is a $C$-way classification model, given a correctly classified example $\vx_0$ with $F(\vx_0) = y_0$, an adversarial perturbation is defined as $\vdelta\in \R^d$ such that $F(\vx_0+\vdelta) \neq y_0$. 
\begin{definition}[Robustness Verification Problem]
Given $F, \vx_0$ and a perturbation radius $\epsilon$, the robustness verification problem aims to determine whether 
there exists an adversarial example within $\epsilon$ ball around $\vx_0$. Formally, we determine whether 
the following statement is true:  

\begin{equation}
F(\vx_0+\vdelta) = y_0, \ \ \forall \|\vdelta\|_p \leq \epsilon.
    \label{eq:robustness}
\end{equation}
\end{definition}
%
%
Giving the exact ``yes/no'' answer to~\eqref{eq:robustness} is NP-complete for neural networks~\citep{katz2017reluplex} and tree ensembles~\citep{kantchelian2016evasion}. {\bf Adversarial attack} algorithms are developed to find an adverarial perturbation $\vdelta$ that satisfies \eqref{eq:robustness}. 
For example, several widely used attacks have been developed for attacking neural networks~\citep{carlini2017towards,madry2017towards,goodfellow2014explaining} and other general classifiers~\citep{cheng2018queryefficient,chen2019boundary}. 
However, adversarial attacks can only find adversarial examples which do not provide a {\bf sound} safety guarantee --- even if an attack fails to find an adversarial example, it does not imply no adversarial example exists.

Robustness verification algorithms aim to find a {\bf sound} solution to \eqref{eq:robustness} --- they output yes for a subset of yes instances of \eqref{eq:robustness}.  However they may not be {\bf complete}, in the sense that it may not be able to answer yes for all the yes instances of  \eqref{eq:robustness}. 
Therefore we will refer solving \eqref{eq:robustness} exactly as the ``complete verification problem'', while in general a verification algorithm can be incomplete\footnote{In some works, incomplete verification is referred to as ``approximate'' verification where the goal is to guarantee a lower bound for the norm of the minimum adversarial example, or ``relaxed'' verification emphasizing the relaxation techniques used to solving an optimization problem related to~\eqref{eq:robustness}.} (providing a sound but incomplete solution to \eqref{eq:robustness}). 
Below we will review existing works on verification and their connections to certified defense. 

\paragraph{Robustness verification}
For neural network, it has been shown complete verification is NP-complete for ReLU networks, so many recent works have been focusing on developing efficient (but incomplete)
 robustness verification algorithms~\citep{kolter2017provable,zhang2018efficient,weng2018towards,singh2018fast,wang2018efficient,singh2019abstract,dvijotham2018dual}. Many of them follow the linear or convex relaxation based approach~\citep{salman2019convex}, where~\eqref{eq:robustness} is solved as an optimization problem with relaxed constraints.  However, since ensemble trees are discrete step functions, none of these neural network verification algorithms can be effectively applied.

Specialized algorithms are required for verifying tree ensembles. 
\citet{kantchelian2016evasion} first  showed that complete verification for ensemble tree is NP-complete when there are multiple trees with depth $\geq 2$. An integer programming method was proposed for complete verification which requires exponential time. 
Later on, a single decision tree is verified for evaluating robustness of an RL policy in \citep{bastani2018verifiable}.
 More recently, \citet{chen2019robustness} gave a comprehensive study on the robustness of tree ensemble models;~\citet{ranzato2020abstract} and~\citet{ranzato2019robustness} proposed a tree ensemble robustness and stability verification method based on abstract interpretation; and ~\citet{tornblom2019abstraction} introduced an abstraction-refinement procedure which iteratively refines a partition of the input space.
However, all these previous works only consider $\ell_\infty$ perturbation model (i.e., setting the norm to be $\|\vdelta\|_\infty$ in \eqref{eq:robustness}). The $\ell_\infty$ norm assumption makes verification much easier on decision trees and stumps as perturbations can be considered independently across features, aligning with the decision procedure of tree based models.
 
\paragraph{Certified Defense} Many approaches have been proposed to improve the robustness of a classifier, however evaluating a defense method is often tricky. Many works evaluate model robustness based on {\bf empirical robust accuracy}, defined as the percentage of correctly classified samples under a specific set of attacks within a predefined threat model (e.g., an $\ell_p$ $\epsilon$-ball)~\citep{madry2017towards,chen2019tree}. However, using such measurement can lead to a false sense of robustness~\cite{athalye2018obfuscated}, since robustness against a specific kind of attack doesn't give a sound solution to \eqref{eq:robustness}. 
In fact, many proposed empirical defense algorithms were broken under more sophisticated attacks~\cite{athalye2018obfuscated,tramer2020adaptive}. Instead, certified adversarial defense algorithms evaluate the classifier based on {\bf certified robust accuracy}, defined as the percentage of correctly classified samples for which the robustness can be verified within the $\epsilon$ ball. 
Most of the certified defense algorithms are based on finding the weights to minimize the certified robust loss measured by some robustness verification algorithms~\cite{kolter2017provable,wong2018scaling,wang2018mixtrain,mirman2018differentiable,zhang2019towards}.

Several recent works studied robust tree based models. In \cite{chen2019tree}, an adversarial training approach is proposed to improve $\ell_\infty$ norm robustness of random forest and GBDT. \citet{wang2019tree} proposed another empirical defense also for $\ell_\infty$ norm robustness. The only certified defense that can provide provable robustness guarantees is given in \cite{andriushchenko2019provably}, where they 
proposed a boosting algorithm to improve the {\bf certified robust error} of ensemble trees and stumps with respect to $\ell_\infty$ norm perturbation. This method cannot be directly extended to $\ell_p$ norm perturbations since it relies on independence between features: when one feature is perturbed, the perturbations of other features are irrelevant.

\section{$\ell_p$-norm Robustness Verification of Stumps and Trees}
\label{sec:verification}

The robustness verification problem for ensemble trees and stumps requires us to solve~\eqref{eq:robustness} given a model $F(\cdot)$.  For some of the cases, we will show that computing~\eqref{eq:robustness} exactly (complete robustness verification) is NP-complete, so in those cases we will propose efficient polynomial time algorithms for computing a sound but incomplete solution to the robustness verification problem.  

\paragraph{Summary of our results}
For a single decision tree, \citet{chen2019robustness} shows that $\ell_\infty$ robustness can be evaluated in linear time. We show that their algorithm can be extended to the $\ell_p$ norm case for $p\in[0,\infty]$. Furthermore, we can also extend the multi-level $\ell_{\infty}$ verification framework~\citep{chen2019robustness} for tree ensembles to general $\ell_p$ cases, allowing efficient and sound verification for general $\ell_p$ norm. For evaluating the robustness of an ensemble decision stump, \citet{andriushchenko2019provably} showed that the $\ell_\infty$ case can be solved in polynomial time, but their algorithm uses the fact that features are uncorrelated under $\ell_\infty$ norm perturbations so cannot be used for any $p<\infty$ case. We prove that the $\ell_0$ norm robustness evaluation can be done in linear time, while for the $\ell_p$ norm case with $p\in (0, \infty)$, the robustness verification problem is   NP-complete. We then propose an efficient dynamic programming algorithm to obtain a good lower bound for verification. 


\subsection{A single decision tree}
\label{sec:single_tree}
We first consider the simple case of a single decision tree. Assume the decision tree has $n$ leaf nodes and for a given example $x$ with $d$ features, starting from the root, $x$ traverses the intermediate tree
levels until reaching a leaf node. Each internal node $i$ determines whether $x$ will be passed to left or right child by checking $\rmI(x_{t_i} > \eta_i)$, where $t_i$ is the feature to spilt at in node $i$ and $\eta_i$ is the threshold. Each leaf node $v_i$ has a value $v_i$ indicating the prediction value of the tree. 

If we define $B^i$ as the set of input $x$ that can reach leaf node $i$, due to the decision tree structure, $B^i$ can be represented as a $d$-dimensional box: 
\begin{equation}
    B^i=(l^i_1, r^i_1] \times \cdots \times (l^i_d, r^i_d]. 
\end{equation}
Some of the $l, r$ can be $-\infty$ or $+\infty$. As discussed in Section 3.1 of \citep{chen2019robustness}, the box can be computed efficiently in linear time by traversing the tree. To certify whether there exists any misclassified points under perturbation  $\|\delta\|_p \leq \epsilon$, we can enumerate boxes for all $n$ leaf nodes and check the minimum distance from $\vx_0$ to each box. The following proposition  shows that the $\ell_p$  norm distance between a point and a box can be computed in $O(d)$ time, and thus the complete robustness verification problem for a single tree can be solved in $O(dn)$ time. 
\begin{proposition}
\label{thm:point_box_dist}
Given a box $B=(l_1, r_1]\times \cdots  \times (l_d, r_d]$ and a point $\vx \in \mathbb{R}^d$. The minimum $\ell_p$ distance ($p \in [0, \infty]$) from $x$ to $B$ is $\|z - x\|_p$ where:
\begin{equation}
z_i = \begin{cases}
x_i, & l_i \leq x_i \leq u_i\\
l_i, & x_i < l_i \\
u_i, & x_i > u_i.
\end{cases}
\label{eq:shortest_box}
\end{equation}
\end{proposition}
We define the operator $\text{dist}_p(B, x)$ to be the minimum $\ell_p$ distance between $x$ to a box $B$. We define the $\ell_p$ norm ball $\text{Ball}_p(x, \epsilon) = \{ x^\prime |  \| x^\prime - x \|_p \leq \epsilon \}$, and we use $\cap$ to denote the intersection between a $\ell_p$ ball and a box. $B \cap \text{Ball}_p(x, \epsilon) \neq \emptyset$ if and only if $\text{dist}_p(B, x) \leq \epsilon$. 
 
%





\subsection{Ensemble decision stumps}
\label{sec:stump_verification}

A decision stump is a decision tree with only one root node and two leaf nodes. We assume there are $T$ decision stumps and the $i$-th  decision stump gives the prediction
\begin{equation*}
    f^i(x) = \begin{cases}
    w^i_l &\text{ if } x_{t_i}<\eta^i \\
    w^i_r &\text{ if } x_{t_i}\geq \eta^i.
    \end{cases}
\end{equation*}
The prediction of a decision stump ensemble $F(x) = \sum_i f^i(x)$ can be decomposed into each feature in the following way. For each feature $j$, assume $j_1, \dots, j_{T_j}$ are the decision stumps using feature $j$, we can collect  all the thresholds $[\eta^{j_1}, \dots, \eta^{j_{T_j}}]$. Without loss of generality, assume $\eta^{j_1} \leq \dots \leq \eta^{j_{T_j}}$ then the prediction values assigned in each interval can be denoted as
\begin{equation}
    g^j(x_j) = v^{j_t}  \ \ \text{ if } \eta^{j_t} < x_j \leq \eta^{j_{t+1}}
    \label{eq:def_gx}
\end{equation}
where 
\begin{equation*}
    v^{j_t} = w_l^{j_1} + \dots + w^{j_t}_l + w^{j_{t+1}}_r + \dots + w^{j_{T_j}}_r,
\end{equation*}
and $x_j$ is the value of sample $x$ on feature $j$. The overall prediction can be written as the summation over the predicted values of each feature: 
\begin{equation}
    F(x) = \sum_{j=1}^d g^j(x_j), 
    \label{eq:separate}
\end{equation}
 and the final prediction is given by $y = sgn(F(x))$.

\paragraph{$\ell_0$ ensemble stump verification}

Assume $F(x)$ is originally positive and we want to make it as small as possible by perturbing $\delta$ features (in this case, $\delta$ should be a positive integer). For each feature $j$, we want to know {\it the maximum decrease of prediction value by  changing this feature,}
which can be computed as
\begin{equation}
    c^j = \min_t v^{j_t} - g^j(x_j),
\end{equation}
and we should choose $\delta$ features with smallest $c^j$ values to perturb. Let $S_\delta$ denotes the set with  $\delta$ smallest $c^j$ values, we have
\begin{equation}
    \min_{\|x-x'\|_0\leq K  } F(x') = F(x) + \sum_{i\in S_K} c^j. 
\end{equation}
Therefore verification can be done exactly in $O(T+d\text{log}(d))$ time, where $O(d\text{log}(d))$ is the cost of sorting $d$ values $\{c^1, ..., c^d\}$.

\paragraph{$\ell_p$ ensemble stump verification} 
The difficulty of $\ell_p$ norm robustness verification is that the perturbations on each feature are correlated, so we can't separate all the features as in \cite{andriushchenko2019provably} for the $\ell_\infty$ norm case. 
In the following, we prove that the complete $\ell_p$ norm verification is NP-complete by showing a reduction from Knapsack to $\ell_p$ norm ensemble stump verification. This shows that $\ell_p$ norm verification can belong to a different complexity class compared to the $\ell_\infty$ norm case. 

\begin{theorem}
Solving $\ell_p$ norm robustness verification (with soundness and completeness) as in Eq.~\eqref{eq:robustness} for an ensemble decision stumps is NP-complete when $p\in (0, \infty)$. 
\end{theorem}
\begin{proof}
We show that a 0-1 Knapsack problem can be reduced to an ensemble stump verification problem. 
A 0-1 Knapsack problem can be defined as follows. Assume there are $T$ items each with weight $w_i$ and value $v_i$, the (decision version of) 0-1 Knapsack problem aims to determine whether there exists a subset of items $S$ such that $\sum_{i\in S} w_i \leq C$ and with value 
$\sum_{i\in S}v_i \geq D$. 

Now we construct a decision stump verification problem with $T$ features and $T$ stumps from the 0-1 Knapsack problem, where each decision stump corresponds to one feature. Assume $x$ is the original example, we define each decision stump to be
\begin{equation}
    g^i(s) = -v_i I(s>\eta_i)+\frac{D}{T}, \ \ \ \text{ where } \eta_i = x_i + w_i^{(1/p)}, 
\end{equation}
where $I(\cdot)$ is the indicator function. The goal is to verify $\ell_p$ robustness with $\epsilon = C^{(1/p)}$.
We need to show that this robustness verification problem outputs YES ($\min_{\|x-x'\|_p\leq \epsilon} \sum_i g^i(x'_i) < 0 $) if and only if the Knapsack solution is also YES. 
If the verification found $v^* = \min_{\|x-x'\|_p\leq \epsilon} \sum_i g^i(x'_i)<0$, 
let $x'$ be the corresponding solution of verification, then  we can choose the following $S$ for  0-1 Knapsack: 
\begin{equation}
    S = \{i \mid x_i' > \eta_i\}
\end{equation}
It is guaranteed that 
\begin{equation}
    \sum_{i\in S} w_i = \sum_{i\in S} |\eta_i-x_i|^p \leq \sum_i |x'_i - x_i|^p \leq \epsilon^p = C 
    \end{equation}
    and by the definition of $g^i$ we have $\sum_i g^i(x'_i) = D-  \sum_{i\in S} v_i\leq 0$, so this subset $S$ will also be feasible for the Knapsack problem. On the other hand, if
    the 0-1 Knapsack problem has a solution $S$, 
    for robustness verification problem we can choose $x'$ such that 
 %
    \begin{equation*}
        x'_i = \begin{cases}
        \eta_i & \text{ if }  i\in S \\
        x_i &\text{ otherwise }  
        \end{cases}
    \end{equation*}
    By definition we have $\sum_{i} g^i(x'_i) = D - \sum_{i\in S}v_i <0$. 
    Therefore the Knapsack problem, which is NP-complete, can be reduced to  $\ell_p$ norm  decision stump verification problem with any $p\in (0, \infty)$ in polynomial time. 
\end{proof}


\paragraph{Incomplete Verification for $\ell_p$ robustness}
Although it's impossible to solve $\ell_p$ verification for decision stumps in polynomial time, we show sound verification can be done in polynomial time by dynamic programming, inspired by the pseudo-polynomial time algorithm for Knapsack. 

Let $\eta^{j_1}, \dots, \eta^{j_{T_j}}$ be the thresholds for feature $j$ and $v^{j_1}, \dots, v^{j_{T_j}}$ be the corresponding values, our dynamic programming maintains the following value for each $\epsilon$: ``given maximal $\epsilon$ perturbation to the first $j$ features, what's the minimal prediction of the perturbed $x$''. We denote this value as $D(\epsilon, j)$, then the following recursion holds:
\begin{equation*} 
\setlength\abovedisplayskip{6pt}
    D(\epsilon, j+1) = \min_{\delta \in [0, \epsilon] } D(\epsilon - \delta, j ) + C(\delta, j+1), 
\end{equation*}
where $C(\delta, j+1):= \min_{|x_j' - x_j|<\delta} g^j(x'_j)$ which can be pre-computed. Note that $\delta, \epsilon$ can be real numbers so exactly running this DP requires exponential time. Our approximate algorithm allows $\epsilon, \delta$ only up to certain precision. If we choose precision $\nu$, then 
we only consider values $\nu, 2\nu, \dots, P\nu$ (the smallest $P$ with $P\nu > \epsilon$). 
To ensure the verification algorithm is sound,  
the recursion will become
\begin{equation} 
    \tilde{D}(a\nu, j+1) = \min_{b\in \{1, \dots, a\} } \tilde{D}( (a - b+1)\nu, j ) + C(b\nu, j+1), 
    \label{eq:dp_verification}
\end{equation}
and the final solution should be $\tilde{D}(\ceil{\epsilon}, d)$  where $\ceil{\epsilon}:=P\nu$ means rounding $\epsilon$ up to the closest grid. Note that the $+1$ term in the recursion is to ensure that the resulting value is a lower bound of the original solution. 
The verification algorithm can verify a sample in $O(Pd + T)$ time
, in which $d$ is  dimension and $P$ is the  number of discretizations.


 \subsection{$\ell_p$ norm verification for ensemble decision trees}

\citet{kantchelian2016evasion} showed that for general ensemble trees, complete $\ell_\infty$ robustness verification can formulated as a mixed integer linear programming problem, which is NP-Complete, and \citet{chen2019robustness} proposed a fast polynomial time hierarchical verification framework to verify the model to a desired precision. For a tree ensemble with $T$ trees and an input example $x$, \citet{chen2019robustness} first check all the leaf nodes of each tree and only keep the leaf nodes that $x$ can reach under the given perturbation. In the $\ell_\infty$ case, both the perturbation ball of $x$ and the decision boundary of a leaf node can be represented as boxes (see Sec.~\ref{sec:single_tree}), therefore it is easy to check whether the two boxes have an intersection. 
Then $T$ trees are splited into $\frac{T}{K}$ groups, each with $K$ trees. Trees from different groups are considered independently; the $K$ trees within a group form a graph where each size-$K$ clique in this graph represents a possible prediction value of all trees within this group given $\ell_\infty$ input perturbation. Enumerating all size-$K$ cliques allows us to obtain the worst case prediction of the $K$ trees within a group, and then we can combine the worst case predictions of all $\frac{T}{K}$ groups (e.g., directly adding all of them) to obtain an over-estimated worst case prediction of the entire ensemble. The results can be tightened by considering each group as a ``virtual tree'' and merge virtual trees into a new level of groups.

The most important procedure in~\citep{chen2019robustness} is to check whether a set of leaf nodes from different trees within a group can form a valid size-$K$ clique, which involves checking the intersections among the decision boundaries of leaf nodes from different trees and the intersection among the clique and the perturbation ball. We extend this procedure to $\ell_p$ setting in our work following two steps:

First, we check the intersection between input perturbation $\text{Ball}_p(x, \epsilon)$ and a box $B^i$ using Proposition~\ref{thm:point_box_dist}. Initially, we only consider the set of leaf node that has $\text{dist}_p(B^i, \vx) \leq \epsilon$ ($B^i$ is the decision boundary of a leaf). 

Second, in $\ell_\infty$ case, since the $\ell_\infty$ perturbation ball is also a box, it is possible to use the boxicity property to obtain intersections which are represented as size-$K$ cliques in~\citet{chen2019robustness}.
This boxicity property is not hold anymore for general $\ell_p$ input perturbations. \citet{chen2019robustness} showed that for a set of $\ell_\infty$ boxes $\{B^1, \dots, B^T\}$, if $B^i \cap B^j \neq \emptyset$ for all $i,j$ $(i\neq j)$, and $B^i \cap \text{Ball}_\infty(x, \epsilon) \neq \emptyset$ for all $i$, then it guarantees that $ B^1 \cap B^2 ... \cap B^T \cap \text{Ball}_\infty(x, \epsilon) \neq \emptyset$. However, for $\ell_p$ ($p \neq \infty$) norm perturbation, under the same condition cannot guarantee that $ B^1 \cap B^2 ... \cap B^T \cap \text{Ball}_p(x, \epsilon) \neq \emptyset$. In fact, even if $\text{Ball}_p(x, \epsilon)\cap B^t\neq \emptyset$ for any $t$, $ B^1 \cap B^2 ... \cap B^T \cap \text{Ball}_\infty(x, \epsilon)$ can still be empty. A counter example with $\ell_1$ is shown in Figure~\ref{fig:l1_case} and similar counter examples can be found for any $p<\infty$.
\begin{figure}[ht]
\centering
		\includegraphics[width=0.75\linewidth]{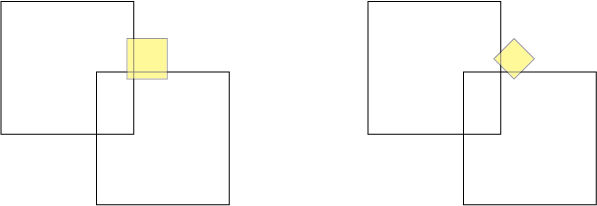}
		\caption{In the $\ell_p$ case, the perturbation ball is not a box and the general $\ell_p$ version of the Lemma 1 in~\citep{chen2019robustness} is not true. Here we present a counter example in $\ell_1$.}
		\label{fig:l1_case}
	\end{figure}

Therefore, we need to check whether $\bar{B} := B^1 \cap \dots \cap B^T$, which is still a box, has nonempty intersection with input perturbation $\text{Ball}_p(x, \epsilon)$. This step can be computed using Proposition \ref{thm:point_box_dist}, which costs $O(d)$ time.
After this additional procedure, we can safely generalize the $\ell_{\infty}$ framework to $\ell_p(p \geq 0)$ cases by simply replacing the procedure. We include the detail algorithm for enumerating the size-$K$ cliques in Appendix \ref{alg:clique_enum}.

\section{Training $\ell_p$-robust Boosted Stumps and Trees}
\label{sec:defense}

Based on the general $\ell_p$ verification algorithm for stump ensembles described in Section \ref{sec:stump_verification}, we develop certified defense algorithms for training ensemble stumps and trees. The main challenge is that 
for $\ell_p (p > 0)$, different from the $\ell_{\infty}$ case, the correlation between features should be considered. 
Following the setting in \citep{andriushchenko2019provably}, we use an exponential loss function $L$, where for a point $(x, y) \in \mathbb{R}^d \times \{-1, 1\}$, $L(y f(x)) = \exp{(-yf(x))}$. 
However, our algorithms can be generalized to other strictly monotonic and convex loss functions. We consider each training example $(x, y) \in \mathbb{S}$ is perturbed in $\text{Ball}_p(x, \epsilon)$, $\mathbb{S}$ is the training set.


\subsection{$\ell_p$ robust boosted stumps}

Given a decision stump ensemble $F(x)=\sum_{i=1}^{T} f_i(x)$ with $T$ stumps, without loss of generality, we assume the first $T-1$ stumps, defined as $F_{T-1}(x)=\sum_{i=1}^{T-1} f_i(x)$, are already trained and fixed, and
our target is to update $F$ with a new stump $f_T(x)$. Here we define a stump as $ f(x) = w_l + \bm{1}_{x_{j} \geq b}w_r$ which splits the space at threshold $b$ on feature $j$ and predict $w_l$ (left leaf prediction) or $w_l + w_r$ (right leaf prediction). Our goal is to select the 4 parameters ($b$, $j$, $w_l$, $w_r$) robustly by minimizing the minimax loss:
\begin{equation}
\setlength\abovedisplayskip{3pt}
\label{eq:robust_objective}
\begin{aligned}
&\min_{j, b, w_l, w_r} \sum_{(x,y) \in \mathbb{S}} \max_{\|\delta\|_p \leq \epsilon}L(y F(x + \delta))
\end{aligned}
\end{equation}
To solve this optimization, we first consider a sub-problem which finds the optimal $w_r^*$ and $w_l^*$ for a fixed split $(j', b')$.
\begin{equation}
\setlength\abovedisplayskip{3pt}
\begin{aligned}
    w_l^*, w_r^* = &\argmin_{w_l, w_r}\!\sum_{(x,y) \in \mathbb{S}}\!\max_{\|\delta\|_p\!\leq \epsilon}\! L(y F(x + \delta))\\
    \text{s.t. }& j = j', b = b'
\end{aligned}
\setlength\belowdisplayskip{3pt}
\label{eq:Lp_stumps_w_optimization}
\end{equation}
For the inner maximization, we note that the loss function is monotonically decreasing, therefore we can replace the maximization as an minimization inside the loss function: 
\begin{align*}
    &\max_{\|\delta\|_p \leq \epsilon}L( y F(x + \delta)) \\
    &=\max_{\|\delta\|_p \leq \epsilon}L \left(y F_{T-1}(x + \delta) + y f_T(x + \delta)\right)\\
    &=L\left(\min_{\|\delta\|_p \leq \epsilon} (y F_{T-1}(x + \delta) + y f_T(x + \delta))\right)\\
    &= 
L\left(\min_{\|\delta\|_p \leq \epsilon} (y F_{T-1}(x + \delta) + y w_l + y w_r \bm{1}_{x_{j'} + \delta_{j'} \geq b'})\right)
\end{align*}
The inner minimization can then be considered as a stump ensemble verification problem. According to Section \ref{sec:stump_verification}, for each $x$, we can derive a lower bound of the inner minimization, denoted as $\Tilde{D}(\ceil{\epsilon}, d)$:
\begin{equation*}
\begin{aligned}
    &\min_{\|\delta\|_p \leq \epsilon} (y F_{T - 1}(x + \delta) + y w_l + y w_r \bm{1}_{x_{j'} + \delta_{j'} \geq b'})\\
    &\geq \Tilde{D}_{(x,y)}(\ceil{\epsilon}, d).
\end{aligned}
\end{equation*}
For simplicity, we omit subscript $(x,y)$ in the analysis below. Our goal is to give $\Tilde{D}(\ceil{\epsilon}, d)$ as a function of $w_l$ and $w_r$. This requires a small extension to the DP based verification algorithm. In~\eqref{eq:dp_verification}, we can consider the $d$ features in any order. We can solve the DP by first solving all other $d-1$ features except $j^\prime$, and obtain $\tilde{D}_{\setminus j^\prime}(a \nu, j)$ for all $a \in \{1, \cdots, P \}$ and $j \in \{1, \cdots. d-1\}$ (we denote the DP table as $\tilde{D}_{\setminus j^\prime}$ to emphasize that it does not include feature $j^\prime$).
$\Tilde{D}_{\backslash j^\prime}(a \nu, d-1)$ is a lower bound of the minimum prediction value under perturbation $a \nu$ excluding all stumps involving feature $j^\prime$.
Then, the recursion for $\Tilde{D}(\ceil{\epsilon}, d)$ needs to consider the minimum of two settings, representing the left or right leaf is selected for the last stump:
\begin{align}
 &\Tilde{D}(\ceil{\epsilon}, d)\!=\! \min\left(\Tilde{D}_L(\ceil{\epsilon}, d), \Tilde{D}_R(\ceil{\epsilon}, d)\right)\nonumber\\
    &\Tilde{D}_L(\ceil{\epsilon},\!d)\!
    =\! \min_{a \in [P]} \left(\Tilde{D}_{\backslash j^\prime}((P\!-\!a\!+\!1)\nu,d\!-\!1)\!+\!C_L(a \nu,j^\prime) \right )\!\nonumber\\
    & \quad\quad\quad\quad\quad +\!y w_l\nonumber\\
    &\Tilde{D}_R(\ceil{\epsilon},d)
= \!\min_{a \in [P]} \left (\Tilde{D}_{\backslash j^\prime}((P\!-\!a\!+\!1)\nu, d\!-\!1)\!+\! C_R(a \nu, j^\prime) \right )\nonumber\\
& \quad\quad\quad\quad\quad +\!y(w_r\!+\!w_l)\nonumber\\
    &C_L(a\nu, j)= \min_{|x_{j} - x_{j}^\prime | \leq a \nu, x_{j}^\prime < b^\prime}g^j(x^\prime)\nonumber\\
    &C_R(a\nu, j)= \min_{|x_{j} - x_{j}^\prime| \leq a \nu, x_{j}' \geq b^\prime}g^j(x^\prime)
    \label{eq:Lp_stumps}
\end{align}
In (\ref{eq:Lp_stumps}), $\Tilde{D}_L(\ceil{\epsilon},\!d)$ and $ \Tilde{D}_R(\ceil{\epsilon},\!d)$ denote the minimum prediction value of the sample $(x, y)$ when perturbed into the left or right side of the split $(j', b')$. $C_L(a \nu, j), C_R(a \nu, j)$ denote the minimum prediction  when $x$ is perturbed into the left or right side of the split on feature $j$ with perturbation $a\nu$,  where $g^j(x)$ is defined as in~\eqref{eq:def_gx} but with the last tree $f_T(x)$ excluded (i.e., computed on $F_{T-1}$).

After obtaining the lower bound of the inner minimization, instead of solving the original optimization (\ref{eq:Lp_stumps_w_optimization}), here we solve
\begin{equation}
    w_l^*, w_r^* = \argmin_{w_l, w_r} \sum_{(x, y) \in \mathcal{S}} L(\Tilde{D}_{(x,y)}(\ceil{\epsilon}, d)).
    \label{eq:Lp_stump_training_optimization}
\end{equation}

\begin{theorem}
$\sum_{(x, y) \in \mathcal{S}} L(\Tilde{D}_{(x,y)}(\ceil{\epsilon}, d))$ defined in \eqref{eq:Lp_stump_training_optimization} is jointly convex in $w_l, w_r$.
\label{thm:convex}
\end{theorem}

The proof can be found in the Appendix~\ref{sec:proof}. Based on this theorem, we can use coordinate descent to solve the minimization: fix $w_r$ and minimize over $w_l$, then fix $w_l$ and minimize over $w_r$ (similar to \citet{andriushchenko2019provably}). For exponential loss, when $w_r$ is fixed, we can use a closed form solution to update $w_l$ (see Appendix~\ref{sec:lp_stump_closeform}). When $w_l$ is fixed, we use bisection to get the optimal $w_r$.
For general loss functions, both $w_l$ and $w_r$ can be solved by bisection.

After estimating $w_l^*, w_r^*$ in (\ref{eq:Lp_stumps_w_optimization}), we can iterate over all the possible split positions $(j, b)$ and select the position with minimum robust loss.
Our proposed general $\ell_p$ norm robust training algorithm for stump ensembles can train a new stump in $O(TN(Pd + T) + dBT)$ time, where $B$ is the number of candidate $b$s and $N$ is the size of dataset. For fixed $j$,  $\epsilon$ and precision, $\Tilde{D}_{\setminus j^\prime}(a \nu, d-1)$ is fixed for all $a\in[P]$ and can be pre-calculated, which costs $O(N(Pd + T))$ time. And in implementation, we only need to calculate $T + 1$ different $\Tilde{D}_{\setminus j^\prime}(a \nu, d-1)$, which costs $O(TN(Pd + T))$ time. After obtaining $\Tilde{D}_{\setminus j^\prime}(a \nu, d-1)$, in each iteration, $\Tilde{D}(\ceil{\epsilon}, d)$ can be derived in $O(T)$ time (despite having $P$ discretizations,  there are only $T$ possible values in the minimization in Eq. \eqref{eq:Lp_stumps}, and an efficient implementation can exploit this fact). The bisection searching for $w_l^*$ and $w_r^*$ can also be finished in $O(1)$ time with fixed parameters. Thus the above algorithm can train a stump ensemble in $O(TN(Pd + T) + dBT)$ time.

\subsection{$\ell_p$ robust boosted trees}
\paragraph{Single decision tree}
\label{sec:tree_training}
Our goal is to solve (\ref{eq:robust_objective}) where $F$ is a single tree. Different from the $\ell_\infty$ case, in $\ell_p$ cases, perturbation on one dimension can reduce the possible perturbation on other dimensions. Therefore, when updating a stump ensemble, perturbation bound $\epsilon$ will be consumed along the trajectory from the tree root to leaf nodes. Because the number of features is typically more than the depth of a decision tree, we use each feature only once along one trajectory on the decision tree.
We define $S = \{(x, y) \in \mathbb{S} | \text{dist}_p(x, B^{N_k})\leq \epsilon\}$ 
as the set of samples that can fall into node $N_k$ under $\ell_p$ norm $\epsilon$ bounded perturbation, and $(N_0, N_1, ..., N_{k - 1})$ as the sequence of nodes on the trajectory from tree root $N_0$ to tree node $N_k$. Each node $N_t$ $(0 \leq t < k)$ contains a split $(j_t, b_t)$ which splits the space on feature $j_t$ at value $b_t$.

In the $\ell_p$ norm case, each example has an unique perturbation budget at node $N_k$, as some of the perturbation budget has been consumed in parent nodes splitting other features.
For each sample $(x, y)$, $\ell_p$ norm bounded perturbation in node $N_k$ can be calculated along the trajectory by $\epsilon(x) = (\epsilon^p - \sum_{t \in E} (x_{j_t} - b_t)^p)^{\frac{1}{p}}$, where $E$ is a subset of the node trajectory in which $x$ and $N_{t + 1}$ are on the different sides of node $N_t$, $\forall t \in E$.
Formally, we can define $E$ as $\{t: t < k - 1,  \bm{1}(x_{j_t} \geq b_t) \neq \bm{1}(N_{t + 1} \geq b_t) \}$, where $N_{t + 1} \geq b_t$ denotes that $x'_{j_t} \geq b_t$, $\forall x' \in B^{N_{t + 1}}$. 
This is different from previous works on $\ell_\infty$ perturbations. Now we consider training the node $N_k$ and get the optimal parameters $(j^*, b^*, w_l^*, w_r^*)$:
\begin{equation}
\label{eq:tree_object}
\begin{aligned}
    & j^*, b^*, w_l^*, w_r^* \\
    &= \argmin_{j, b, w_l, w_r} \sum_{(x, y) \in S} \max_{\|\delta\|_p \leq \epsilon(x)} L(f(x + \delta)y),
\end{aligned}
\end{equation}
where $f(\cdot)$ is a new leaf node $f(x) = \rmI(x \geq b) w_r + w_l$, and when training node $N_k$, we only consider the training examples in $S$. The objective in \eqref{eq:tree_object} is similar to that in \eqref{eq:robust_objective} except that there is only one stump to be trained. Therefore, we can use a similar procedure as in previous section to find the optimal parameters.


\paragraph{Boosted decision tree ensemble}
Given a tree ensemble with $T$ trees $F(x) = \sum_{i = 1}^{T} f_i(x)$, we fix the first $T - 1$ trees and train a node $N$ on the $(T)$-th decision tree $f_T(x)$. The optimization problem will be essentially  the same as Eq.~\eqref{eq:tree_object}, but here for $(x, y) \in S$, we should also consider the first $T - 1$ trees, along with prediction of node $N$:
\begin{equation*}
\begin{aligned}
    &\max_{\|\delta\|_p\leq\!\epsilon(x)}\!L(y F_T(x + \delta)) \\
    &=\!\max_{\|\delta\|_p\leq \epsilon(x)}\!L(y F_{T - 1}(x + \delta)\!+\!y(w_l\!+ \!\bm{1}_{x_{j} + \delta_j\!\geq b} w_r))\\
    &=\!L\left(\!\min_{\|\delta\|_p\leq \epsilon(x)} (y F_{T - 1}(x + \delta)\!+\!y(w_l\!+ \!\bm{1}_{x_{j} + \delta_j\!\geq b} w_r))\right).
\end{aligned}
\end{equation*}
Here $F_{T - 1}(x)$ is the prediction from the ensemble of the first $T - 1$ trees. We further lower bound the minimization:
\begin{equation*}
\begin{split}
&\!\min_{\|\delta\|_p\leq \epsilon(x)} (y F_{T - 1}(x + \delta)\!+\!y(w_l\!+ \!\bm{1}_{x_{j} + \delta_j\!\geq b} w_r))\\
&\geq \min_{\|\delta\|_p\leq \epsilon(x)} y F_{T - 1}(x + \delta)\! + \min_{\|\delta\|_p \leq \epsilon(x)} \!y(w_l\!+ \!\bm{1}_{x_{j} + \delta_j\!\geq b} w_r)
\end{split}
\end{equation*}
The first part is the $\ell_p$ robustness verification for tree ensemble, which is challenging to solve efficiently during training time. Here we apply a relatively loose lower bound of $y F_{T-1}(x)$, where
\begin{equation*}
    \sum_{i = 1}^{T - 1} \min_{\|\delta\|_p \leq \epsilon(x)}(y f(x+\delta)) \leq \min_{\|\delta\|_p \leq \epsilon(x)} y F_{T-1}(x)
\end{equation*}

We simply sum up the worst prediction on each previous tree, which can be easily maintained during training. By doing this relaxation, the problem is reduced to building a single tree to boost the $\ell_p$ norm robustness.



\subsection{$\epsilon$ schedule}

When features are correlated in $\ell_p$ cases, we find that it is important to have an $\epsilon$ schedule during the training process -- the $\epsilon$ increases gradually from small to large, instead of using a fixed large $\epsilon$ in the beginning. 
If one directly uses a large $\epsilon$ in the beginning, the first few stumps will allow too much perturbation and the later stumps tend to allow fewer perturbation, making it harder to explore the correlation between features. In ensemble stump training, we increase the $\epsilon$ when training a new stump, and in ensemble tree training, we increase the $\epsilon$ when height of the tree grows. We also include the choice of $\epsilon$ schedules in Appendix \ref{apd:settings}.
 

\vspace{-5pt}
\section{Experimental Results}
\label{sec:exp}

In this section we empirically test the proposed algorithms for $\ell_p$ robustness verification and training. The code is implemented in Python and all the  experiments are conducted on a machine with 2.7 GHz Intel Core i5 CPU with 8G RAM. Our code is publicly available at 
\href{https://github.com/YihanWang617/On-ell_p-Robustness-of-Ensemble-Stumps-and-Trees}{https://github.com/YihanWang617/On-ell\_p-Robustness-of-Ensemble-Stumps-and-Trees}

\begin{table*}[th!]
\begin{center}
\caption{\textbf{General $\ell_p$-norm ensemble stump verification.} This table reports verified test error (verified err.) and average per sample verification time (avg. time) of each method. For our proposed DP based verification, precision is also reported. For $\ell_0$ verification, we report verified errors with $\epsilon_0=1$ (changing 1 pixels). For $\ell_0$ norm, we also report $r^*$, which is the average the number features that can be perturbed at most while the prediction stays the same. 
} 
\scalebox{0.85}{
\setlength\tabcolsep{1.5pt}
\begin{tabular}{c|c?c|c?c|c|c?c|c?c|c|c}
\hline
    
    \multicolumn{2}{c?}{Dataset}&\multicolumn{2}{c?}{$\ell_1$ MILP (complete) ~}&\multicolumn{3}{c?}{ Ours $\ell_1$ DP approx. (incomplete)}&\multicolumn{2}{c?}{Ours vs. MILP}&\multicolumn{3}{c}{Ours $\ell_0$ (complete) verification}\\
    name &$\epsilon_\infty$&verified err.&avg. time &precision& verified err. &avg. time & MILP/ours &speedup & avg. robust $r^*$&verified err. & avg. time\\\hline
    
    breast-cancer &0.3&10.94\%&.030s&0.01&10.94\%&.00025s&1.00&120X&.04&95.62\%&.0006s\\\hline
    diabetes &0.05&35.06\%&.017s&0.0002&35.06\%&.0004s&1.00&40X&.0&100.0\%&.0005s\\\hline
    Fashion-MNIST shoes  &0.1&10.45\%&.105s&0.005&10.55\%&.0013s&.99&80.8X&2.09&16.35\%&.010s\\\hline
    MNIST 1 vs. 5&0.3&3.30\%&0.11s&0.005&3.35\%&0.0013s&1.00&71X&3.33&4.50\%&.010s\\\hline
    
    MNIST 2 vs. 6&0.3&9.64\%&0.099s&0.005&9.69\%&.0012s&.98&82X&1.22&26.43\%&.012s\\
    \hline

\end{tabular}
}
\label{tab:stump_verification}
\end{center}
\end{table*}

\begin{table*}[th!]
\vspace{-3mm}
\begin{center}
\caption{
\textbf{General $\ell_p$-norm  tree ensemble verification.} We report verified test error (verified err.) and average per-example verification time (avg. time) of each method.  $K$: size of cliques; $L$: number of levels in multi-level verification (defined similarly as in~\citep{chen2019robustness}). Our $\ell_p$ incomplete verification can obtain results very close to complete verification (MILP), with huge speedups. 
} 
\scalebox{0.90}{
\setlength\tabcolsep{1.5pt}
\begin{tabular}{c|c?c|c?c|c|c|c?c|c}
\hline
    
    \multicolumn{2}{c?}{Dataset}&\multicolumn{2}{c?}{$\ell_1$ MILP ~}&\multicolumn{4}{c?}{ Ours $\ell_1$ approx.}&\multicolumn{2}{c}{Ours vs. MILP}\\
    name & $\epsilon$ &verified err.&avg. time &K&L& verified err. &avg. time & ratio of verified err. &speedup \\\hline
    
    breast-cancer &0.3&8.03\%&.036s&3&2&8.03\%&.012s&1.00&3X\\\hline
    diabetes &0.05&33.12\%&.027s&3&2&33.12\%&.012s&1.00&2.25X\\\hline
    Fashion-MNIST shoes &0.1&10\%&.091s&3&2&10\%&.011s&1.00&8.23X\\\hline
    MNIST 1 vs. 5 &0.3&4.20\%&0.088s&3&2&4.20\%&.011s&1.00&8X\\\hline
    
    MNIST 2 vs. 6 &0.3&8.60\%&.098s&3&2&8.80\%&.012s&.98&8.17X\\
    \hline

\end{tabular}
}
\label{tab:tree_verification}
\end{center}
\end{table*}


\begin{table*}[h]
\begin{center}
\caption{
\textbf{
$\ell_1$ robust training for stump ensembles.}
We report standard errors and  $\ell_1$ verified errors 
of our training methods ($\ell_1$ training) versus the previous $\ell_\infty$ training algorithm. 
$\epsilon$ is the perturbation bound  for each dataset. For $\epsilon = 1.0$ in mnist dataset, we train the models using $\epsilon_{\infty} = 0.3$.
Our proposed $\ell_1$ training can significantly reduce the $\ell_1$ verified error, and the previous $\ell_\infty$ approach cannot as it was designed to reduce $\ell_\infty$ error only. We conduct a similar experiment for $\ell_2$ norm in Appendix~\ref{apd:ell_2_training}.
} 
\scalebox{0.85}{
\setlength\tabcolsep{1.5pt}
\begin{tabular}{c|c|c|c?c|c?c|c?c|c}
\hline
    
    \multicolumn{4}{c?}{Dataset}&\multicolumn{2}{c?}{standard training ~}&\multicolumn{2}{c?}{$\ell_{\infty}$ training\cite{andriushchenko2019provably}}&\multicolumn{2}{c}{$\ell_1$ training (ours)}\\
    name&$\epsilon_\infty$&$\epsilon_1$&n. stumps&standard err.&verified err. &standard err.& verified err. & standard err. & verified err. \\\hline
    
    breast-cancer & 0.3 &1.0&20&0.73\%&95.62\%&4.37\%&99.27\%&1.46\%&\bf 35.77\%\\\hline
    diabetes & 0.05 &0.05&20&21.43\%&37.66\%&29.2\%&35.06\%&27.27\%&\bf 31.82\%\\\hline
    \multirow{2}*{Fashion-MNIST shoes}& 0.1 & 0.1 &20&6.60\%&69.85\%&7.50\%&10.45\%&7.10\%&\bf 10.35\%\\
    &0.2& 0.5&40& 5.05\% &87.5\% &9.25\% &57.05\% &12.40\% &\bf 32.20\%\\\hline
    \multirow{2}*{MNIST 1 vs. 5} & 0.3 & 0.3 &20&1.23\%&58.76\%&1.68\%&3.30\%&1.28\%&\bf 2.81\%\\
    & 0.3 & 1.0 &40& 0.59\% & 66.01\% & 1.33\% & 17.46\% & 4.49\% & \bf 16.23\% \\\hline
    
    \multirow{2}*{MNIST 2 vs. 6} & 0.3 & 0.3 &20&3.17\%&92.46\%&4.52\%&9.64\%&3.71\%&\bf 8.24\%\\
    & 0.3& 1.0 &40& 2.81\% & 99.49\% & 3.91\% & 44.22\% & 7.73\% & \bf 33.46\% \\
    \hline

\end{tabular}
}
\label{tab:stump_training}
\end{center}
\end{table*}

\begin{table*}[h]
\vspace{-3mm}
\begin{center}
\caption{
\textbf{$\ell_1$ robust training for tree ensembles.} We report standard and $\ell_1$ robust test error for all the three methods. We also report $\epsilon$  for each dataset, and the number of trees in each ensemble. We also report the results of $\ell_2$ robust training for tree ensembles in Appendix \ref{apd:ell_2_training}.}
\scalebox{0.85}{
\setlength\tabcolsep{1.5pt}
\begin{tabular}{c|c|c|c|c?c|c?c|c?c|c}
\hline
    
    \multicolumn{5}{c?}{Dataset}&\multicolumn{2}{c?}{standard training ~}&\multicolumn{2}{c?}{$\ell_{\infty}$ training~\citep{andriushchenko2019provably}}&\multicolumn{2}{c}{$\ell_1$ training (ours)}\\
    name&$\epsilon_\infty$&$\epsilon_1$&n. trees & depth & standard err.&verified err. &standard err.& verified err. & standard err. & verified err. \\\hline
    
    {Fashion-MNIST shoes} 
    &0.2&0.5&5&5&4.65\%&99.85\%&7.85\%&89.54\%&18.71\%&\bf 65.18\%\\\hline
    {breast-cancer} 
    & 0.3&1.0 & 5 & 5 &  0.73\% & 99.26\% & 0.73\% & 99.63\% & 9.56\% & \bf 47.05\% \\\hline
    {MNIST 1 vs. 5} 
    &0.3&0.8&5&5&0.64\%&97.38\%&0.64\%&64.11\%&4.59\%&\bf 36.23\%\\\hline
    
    {MNIST 2 vs. 6} 
    &0.3 & 0.6&5&5 &4.12\%&100.0\%&1.96\%&52.33\%&7.64\%&\bf 39.67\%\\
    \hline

\end{tabular}
}

\label{tab:tree_training}
\vspace{-5pt}
\end{center}
\end{table*}
\vspace{-5pt}
\subsection{$\ell_p$ stump and tree ensemble verification}

\textbf{$\ell_p$ stump ensemble verification }
We evaluate our incomplete $\ell_p$ verification method for stump ensembles on five real datasets. 
Ensembles are robustly trained using the $\ell_{\infty}$ training procedure proposed in \cite{andriushchenko2019provably}, each of which contains 20 stumps. 

For the $\ell_1$ norm robustness verification problem, we have shown it's NP-complete to conduct complete verification. To demonstrate the tightness and efficiency of the proposed Dynamic Programming (DP) based verification,
we also run the Mixed Integer Linear Programming \cite{kantchelian2016evasion} to conduct complete verification, which can take exponential time. In Table \ref{tab:stump_verification}, we can find that the proposed DP algorithm gives almost exactly the same bound as MILP, while being $50-100$ times faster. 
This speedup guarantees its further applications in certified robust training.



For the $\ell_0$ norm robustness verification, we propose a linearithmic time algorithm for complete verification. The results for $\epsilon_0=1$ (changing only 1 feature) are also reported in Table \ref{tab:stump_verification}. We can observe that the proposed method can conduct complete verification in less than $0.1$ second. We find that some models are not robust to $\ell_0$ perturbations with high verified errors. Since our verification method is complete, these models suffer from adversarial examples that change classification outcome by changing only 1 pixel.

\textbf{$\ell_p$ tree ensemble verification}
We evaluate our incomplete $\ell_p$ verification method for tree ensembles on five real datasets. 
Ensemble models being verified are robustly trained with \cite{andriushchenko2019provably}, each of which contains 20 trees. 

Again, we compare our proposed algorithm with MILP-based complete verification~\cite{kantchelian2016evasion} which can take exponential time to get the exact bound. The results are presented in Table~\ref{tab:tree_verification}, and 
parameters of the proposed method ($K$ and $L$) are also reported.
We observe that the proposed verification method gets very tight verified errors while being much faster than the MILP solver. 

\vspace{-5pt}
\subsection{$\ell_p$ robust stump and tree ensemble training}
\textbf{$\ell_p$ robust stump training}
We evaluate our proposed certified training methods on two small size datasets and three medium-size datasets. All the models are trained with standard training, $\ell_{\infty}$ robust training~\cite{andriushchenko2019provably} and our proposed general $\ell_p$ robust training algorithm (in experiments, we set $p = 1$. We also report the $p = 2$ results in Appendix \ref{apd:ell_2_training}).
Models of the same dataset are trained with the same set of hyperparameters 
(details can be found in the Appendix). We evaluate $\ell_1$ verified test error using MILP.
In our experiments, we choose different $\epsilon_\infty$ and $\epsilon_p$ such that the $\ell_\infty$ and $\ell_p$ perturbation balls do not contain each other. 
Standard error and verified robust test error of each model are reported in Table \ref{tab:stump_training}. We also report $\ell_\infty$ robustness of these models in Appendix \ref{apd:ell_infty_robustness}. We observe that the proposed training method can successfully get a more robust model against $\ell_1$ perturbation compared to the previous $\ell_\infty$-norm only training method. 


\textbf{$\ell_p$ robust tree training}
We evaluate our $\ell_p$ robust training method for trees on subsets of three medium size datasets (dataset statistics can be found in the Appendix). 
We report the results of $\ell_1$ robust training tree ensembles in Tables \ref{tab:tree_training}, and results of $\ell_2$ robust training in Appendix \ref{apd:ell_2_training}. It shows that our algorithm achieves better or at least comparable verified error in most cases. 
\begin{figure}[htb]
    \centering
    \includegraphics[width=6cm]{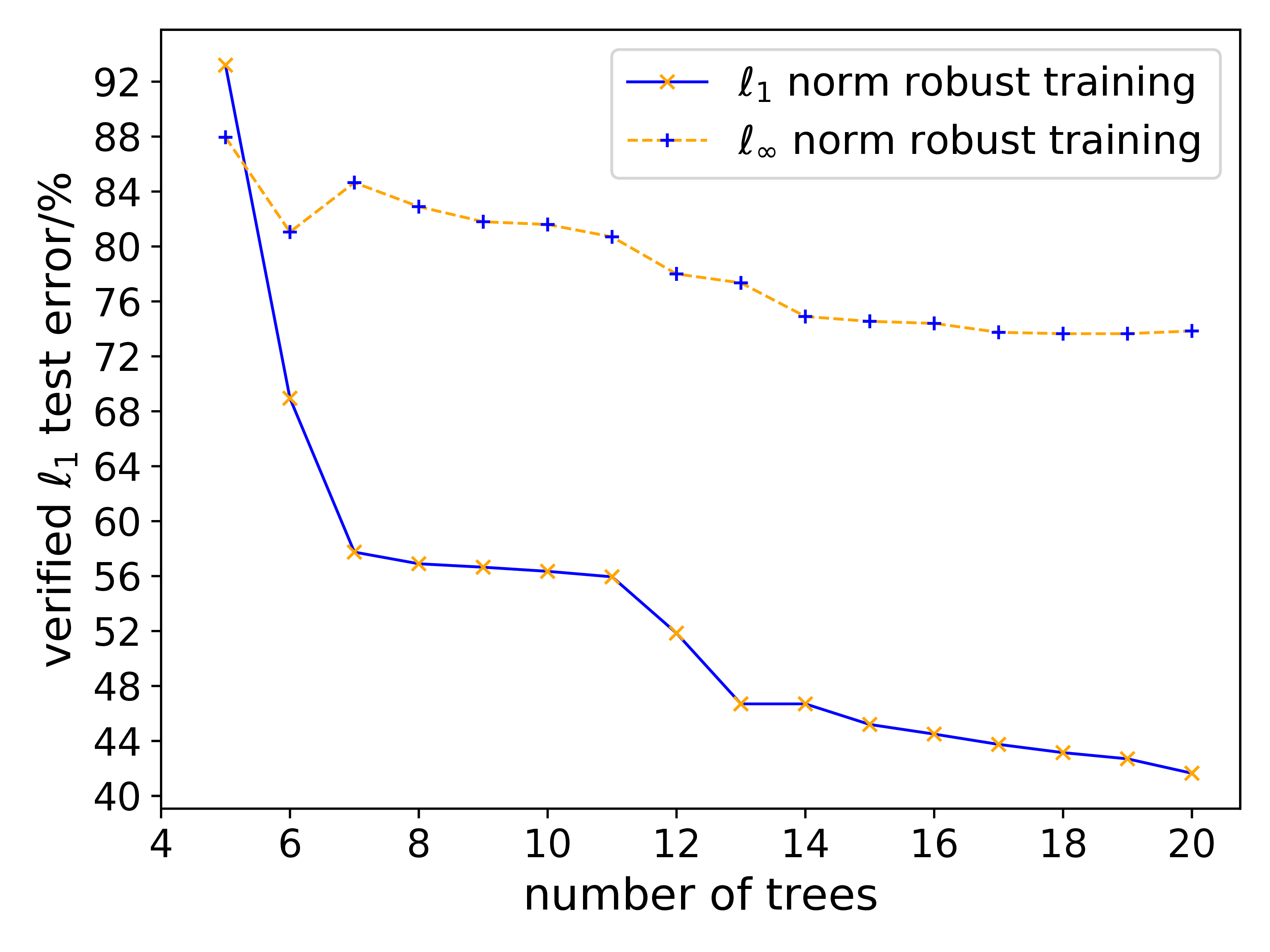}
    \vspace{-8pt}
    \caption{\textbf{$\ell_1$ and $\ell_{\infty}$ robust training on fashion-mnist dataset ($\epsilon_\infty = 0.2$ and $ \epsilon_1 = 0.5$).} We compare verified errors during training when the number of stumps increases.
    }
    \label{fig:robust_error_comparison}
    \vspace{-3mm}
\end{figure}
In addition, we also conduct an example to test the performance of certified training with respect to number of trees. In Figure \ref{fig:robust_error_comparison}, we compare $\ell_\infty$ and $\ell_1$ robust training on fashion-mnist dataset and monitor the performance over the first 20 stumps (the $\epsilon$ scheduling length is 5). We can observe that when number of stumps increases, 
the our $\ell_1$ robust training can indeed gradually reduce $\ell_1$ verified test error, where the $\ell_\infty$ robust training (as a reference) can only slightly improve $\ell_1$ robustness.


\section{Conclusion}
In this paper, we first develop methods to efficiently verify the general $\ell_p$ norm robustness for tree-based ensemble models. Based on our proposed efficient verification algorithms proposed, we further derive the first $\ell_p$ norm certified robust training algorithms for ensemble stumps and trees. 
\clearpage
\section*{Acknowledgement}
We acknowledge Maksym Andriushchenko and Matthias Hein for providing their $\ell_\infty$ certified training code. 
This work is partially supported by NSF IIS-1719097, Intel, Google cloud and Facebook. Huan Zhang is supported by the IBM fellowship. 



\bibliography{paper}
\bibliographystyle{icml2020}

\newpage
\onecolumn
\appendix

\section{Proof of Proposition \ref{thm:point_box_dist}}

\textbf{Proposition 1.} Given a box $B=(l_1, r_1]\times \cdots  \times (l_d, r_d]$ and a point $\vx \in \mathbb{R}^d$. The closest $\ell_p$ distance ($p \in [0, \infty]$) from $x$ to $B$ is $\|z - x\|_p$ where:
\begin{equation*}
z_i = \begin{cases}
x_i, & l_i \leq x_i \leq u_i\\
l_i, & x_i < l_i \\
u_i, & x_i > u_i.
\end{cases}
\end{equation*}

\begin{proof}
For $p > 0$, The goal is to minimize the following objective:
\[
\begin{split}
\min_{z} \|z - x\|_p^p &= \min_{z} \sum_{i=1}^d |z_i - x_i|^p \\
\text{s.t. } & l_i < z_i \leq r_i, \enskip\forall i \in [d].
\end{split}
\]
And for $p = 0$, the objective is 
\begin{equation*}
\begin{split}
\min_{z} \|z - x\|_0 &= \min_{z} \sum_{i=1}^d \rmI(z_i \neq x_i) \\
\text{s.t. } & l_i < z_i \leq r_i, \enskip\forall i \in [d].
\end{split}
\end{equation*}
where $\rmI(\cdot)$ is an indicator function. For $p = \infty$, the objective is 
\begin{equation*}
\begin{split}
\min_{z} \|z - x\|_\infty &= \min_{z} \sum_{i=1}^d | z_i - x_i| \\
\text{s.t. } & l_i < z_i \leq r_i, \enskip\forall i \in [d].
\end{split}
\end{equation*}
Since each term in the summation is separable, we can consider minimizing each term in the summation signs separately. Given the constraints on $z_i$, the minimum is achieved at the condition specified in Eq.~\eqref{eq:shortest_box} regardless of the choice of $p$:
\begin{equation*}
z_i = \begin{cases}
x_i, & l_i \leq x_i \leq u_i\\
l_i, & x_i < l_i \\
u_i, & x_i > u_i.
\end{cases}
\end{equation*}
\end{proof}

\section{Closed form update rule for $\ell_p$ Stump Ensemble Training}
\label{sec:lp_stump_closeform}
For exponential loss we can rewrite eq \eqref{eq:Lp_stump_training_optimization} as 
\begin{equation*}
\begin{aligned}
    \sum_{i = 1}^{N - 1} L(\Tilde{D}(\ceil{\epsilon}, d)) &= \sum_{i = 1}^{N - 1} \gamma_i \exp(-y_i w_l)\\
    &=\sum_{y_i = 1}\gamma_i \exp(-w_l)\!+\!\sum_{y_i = -1} \gamma_i \exp(w_l)\\
\end{aligned}
\end{equation*}
where 
\begin{equation*}
    \gamma_i = L(\Tilde{D}(\ceil{\epsilon}, d) - y_i w_l)
\end{equation*}
which is fixed with a fixed $w_r$.

And we can further derive the optimal $w_l$ at each update step
\begin{equation*}
\begin{aligned}
    \sum_{y_i = 1} \gamma_i (-\exp(-w_l^*)) &+ \sum_{y_i = -1} \gamma_i \exp(w_l^*) = 0\\
    \sum_{y_i = 1} \gamma_i \exp(-w_l^*) &= \sum_{y_i = -1} \gamma_i \exp(w_l^*)\\
    w_l^* &= \ln{\frac{\sum_{y_i = 1} \gamma_i}{\sum_{y_i = -1} \gamma_i}}/2.
\end{aligned}
\end{equation*}

\section{Robustness verification for ensemble trees}
In this section, we provide the detail algorithm of robustness verification for ensemble trees. This algorithm is based on the robustness verification framework in \cite{chen2019robustness}. In Algorithm \ref{alg:clique_enum}, we describe the modified function \texttt{CliqueEnumerate}, which is the key procedure of this framework. The main difference is that after we form the initial set of cliques, we will recheck whether the formed cliques have intersection with the $\ell_p$ perturbation ball (line 18 to 22).

\begin{algorithm*}[htb]
\SetKwInOut{Input}{input}
\Input{$V_1,\ V_2,\ ,\dots,\ V_K$ are the $K$ independent sets (``parts'') of a $K$-partite graph; the graph is defined similarly as in \citet{chen2019robustness}.}
\For{$k\leftarrow 1,\ 2,\ 3,\ \dots,\ K$}{
$U_k\leftarrow \{(A_i,\ B^{i^{(k)}})|i^{(k)}\in V_k,\ A_i=\{i^{(k)}\}\}$\;
\tcc{\small{$U$ is a set of tuples $(A,B)$, which stores a set of cliques and their corresponding boxes. $A$ is the set of nodes in one clique and $B$ is the corresponding box of this clique. Initially, each node in $V_k$ forms a 1-clique itself.}}
}
\CliqueEnum{$U_1,\ U_2,\ ,\dots,\ U_K$}\;

\Fn{\CliqueEnum{$U_1,\ U_2,\ ,\dots,\ U_K$}}{
$\hat{U}_\text{old}\leftarrow U_1$\;

\For{$k\leftarrow 2,\ 3,\ \dots,\ K$}{
$\hat{U}_\text{new}\leftarrow \emptyset$\;

\For{$(\hat{A},\ \hat{B})\in \hat{U}_\text{old}$}{

\For{$(A,\ B)\in U_{k}$}{
\uIf{$B\cap \hat{B}\neq\emptyset$}{
\tcc{\small{A $k$-clique is found; add it as a pseudo node with the intersection of two boxes.}}
$\hat{U}_\text{new}\leftarrow \hat{U}_\text{new}\cup\{(A\cup\hat{A},\ B\cap \hat{B})\}$\;
}
}
}
$\hat{U}_\text{old}\leftarrow \hat{U}_\text{new}$\;
}
$\hat{U} \leftarrow \emptyset$

\For{$(A, B) \in \hat{U}_\text{new}$}{
    \uIf{\CheckClique{$B, d, p, \epsilon_p$}}{
    \tcc{\small{After finding all the $k$-cliques, we need to recheck whether these cliques have intersection with the $\ell_p$ perturbation ball around the example $x$.}}
        $\hat{U} \leftarrow \hat{U} \cup \{(A, B)\}$
    }
}
return $\hat{U}$\;
}

\Fn{\CheckClique{$B$, $d$, $p$, $\epsilon$}}{
$dist \leftarrow \min_{z \in B} \lVert z - x \rVert_p^p$ using Proposition \ref{thm:point_box_dist}

\uIf{$dist < \epsilon^p $}{
return false
}
return true
}

\caption{Enumerating all $K$-cliques on a $K$-partite graph with $\epsilon_p$, dimension $d$ and example $x$}
\label{alg:clique_enum}
\end{algorithm*}

\section{Proof of Theorem \ref{thm:convex}
}
\label{sec:proof}
\begin{proof}
By definition, we have
\begin{equation*}
\begin{aligned}
    L(\Tilde{D}(\ceil{\epsilon}, d)) &=  L(\min(\Tilde{D}_L(\ceil{\epsilon},\!d), \Tilde{D}_R(\ceil{\epsilon},\!d)))\\
    &= \max (L(\Tilde{D}_L(\ceil{\epsilon},\!d)), L(\Tilde{D}_R(\ceil{\epsilon},\!d))).
\end{aligned}
\end{equation*}
Exponential loss $L$ is convex and monotonically increasing; $L(\Tilde{D}_L(\ceil{\epsilon},\!d))$ and $L(\Tilde{D}_R(\ceil{\epsilon},\!d))$ are both jointly convex in $w_l, w_r$. Note that the dynamic programming related terms become constants after they are computed, so they are irrelevant to $w_l, w_r$. Therefore, $L(\Tilde{D}(\ceil{\epsilon}, d))$ and further $\sum_{i = 0}^{N - 1} L(\Tilde{D}(\ceil{\epsilon}, d))$ are jointly convex in $w_l, w_r$.
\end{proof}

\vspace{-15pt}
\begin{table*}[htb]
\begin{center}
\scalebox{0.85}{
\setlength\tabcolsep{1.5pt}
\begin{tabular}{c?c?c?c}
\hline
    
    \multirow{2}*{Dataset}&\multirow{2}*{ensemble stumps lr. }&\multirow{2}*{ensemble trees lr.}&{$\ell_1$ training}\\
    & & & ensemble trees sample size\\ \hline
    
    breast-cancer  &0.4&-&-\\
    diabetes &0.4&-&-\\
    {Fashion-MNIST shoes} & 0.4& 1.0 & 5000 \\
    {MNIST 1 vs. 5} &  0.4 & 1.0 &5000\\
    
    {MNIST 2 vs. 6}  & 0.4 &1.0&5000\\
    \hline

\end{tabular}
}
\caption{
\textbf{Detail settings of the experiments}. Here we report the learning rate of different training methods for ensemble stumps and trees. We also report the sample size in experiments for ensemble tree training and the scheduling length in $\ell_p$ robust training for ensemble stumps.
} 
\label{tab:exp_parameters}
\vspace{-5pt}
\end{center}
\end{table*}

\subsection{Detail settings of the experiments}
\label{apd:settings}
Here we report the detail settings of our experiments in Table \ref{tab:exp_parameters}. For most of the experiments, we follow the learning rate settings in \cite{andriushchenko2019provably}. For $\epsilon$ scheduling length, we empirically set to the best value near $\epsilon_p/\epsilon_{std}$ for each dataset and $\epsilon$ settings (e.g., for $\ell_1$ norm training, the best schedule length is among 2, 3 and 4 epochs for $\epsilon_1 = 1.0$ and $\epsilon_{std} = 0.3$). Here the $\epsilon_{std}$ is $\epsilon_\infty$ used in \cite{andriushchenko2019provably}. 
For each dataset, different methods are trained with the same group of parameters.

For $\ell_1$ robust training for ensemble trees, we use a subsample of training datasets to reduce training time. 
On Fashion-MNIST shoes, MNIST 1 vs. 5 and MNIST 2 vs. 6 datasets, we subsample 5000 images of the selected classes from the original dataset. For $\ell_2$ robust training, we subsample 1000 images of the selected classes from the original dataset.

\subsection{$\ell_{\infty}$ vs. $\ell_p$ robust training}
For a binary classifier $y = \text{sgn}(F(x))$, and a fixed $\epsilon$, we have $\min_{\|\delta\|_p \leq \epsilon} yF(x + \delta) \geq \min_{\|\delta\|_{\infty} \leq \epsilon} yF(x + \delta)$. Therefore, the exact $\ell_{\infty}$ robust loss can be a natural upper bound of $\ell_p$ robust loss. This explains the close result from $\ell_{\infty}$ and $\ell_p$ robust training, when using the same $\epsilon$. However, this $\ell_{\infty}$ upper bound tends to hurt the clean accuracy , which we can see from Table $\ref{tab:stump_training}$. Additionally, unlike $\ell_1$ or $\ell_2$ norms, it is impossible to set this $\ell_\infty$ perturbation to a large value (e.g., $\epsilon_\infty=1.0$).
\section{Additional experiment results}
\subsection{Comparison of $\ell_\infty$ robustness}
\label{apd:ell_infty_robustness}
In this section, we report the $\ell_\infty$ verified errors of models in Table \ref{tab:stump_training}. For each model in the table, we verify the models using $\ell_\infty$ robustness verification of decision stumps~\citep{andriushchenko2019provably} with perturbation norm $\epsilon_\infty$. In general, \citet{andriushchenko2019provably} produces better $\ell_\infty$ norm verification error because it is designed for that case, but when training using our $\ell_1$ robust training procedure with a larger $\ell_1$, models also get relatively good $\ell_\infty$ robustness.
Note that here we train different number of stumps for different $\epsilon_1$(e.g. For MNIST dataset, we train 20 stumps for $\epsilon_1 = 0.3$ and 40 stumps for $\epsilon_1 = 1.0$).
And for a fixed $\epsilon$, we train the $\ell_\infty$ robust model with the same number of stumps with other methods when making comparisons.

\begin{table*}[htb]
\begin{center}
\scalebox{0.90}{
\setlength\tabcolsep{1.5pt}
\begin{tabular}{c|c|c?c?c?c}
\hline
    
    \multicolumn{3}{c?}{Dataset}&{standard training ~}&{$\ell_{\infty}$ training}&{$\ell_1$ training}\\
    name&$\epsilon_\infty$&$\epsilon_1$& $\ell_\infty$ verified err. & $\ell_\infty$ verified err. & $\ell_\infty$ verified err. \\\hline
    
    breast-cancer & 0.3 &1.0&88.32\%&10.94\%&17.51\%\\\hline
    diabetes & 0.05 &0.05&42.85\%&35.06\%&31.81\%\\\hline
    \multirow{2}*{Fashion-MNIST shoes}& 0.1 & 0.1 &69.85\%&11.35\%&11.75\%\\
    &0.2& 0.5& 98.85\% & 19.30\% & 27.60\% \\\hline
    \multirow{2}*{MNIST 1 vs. 5} & 0.3 & 0.3 &67.09\%&4.09\%&4.05\%\\
    & 0.3 & 1.0 & 66.20\% & 3.60\% & 11.59\% \\\hline
    
    \multirow{2}*{MNIST 2 vs. 6} & 0.3 & 0.3 &97.74\%&8.63\%&9.10\%\\
    & 0.3& 1.0 & 100.0\% & 8.69\% & 15.28\% \\
    \hline

\end{tabular}
}
\caption{
\textbf{
$\ell_\infty$ robustness of ensemble decision stumps.}
This table reports the $\ell_\infty$ robustness for the same set of models in Table \ref{tab:stump_training}. For each dataset, we evaluate standard models, the $\ell_\infty$ robust models trained using \citep{andriushchenko2019provably} with perturbation norm $\epsilon_\infty$, and our $\ell_p$ robust model with $p = 1$ and perturbation norm $\epsilon_1$. We test the models with $\ell_\infty$ norm perturbation $\epsilon_\infty$. Standard test errors are omitted as they as the same as in Table~\ref{tab:stump_training}.
} 
\label{tab:ell_infty_comparison}
\end{center}
\end{table*}

\subsection{$\ell_2$ robust training}
\label{apd:ell_2_training}
In Section~\ref{sec:exp} we mainly presented results for the $p=1$ setting, however our robust training procedure works for general $\ell_p$ norm. In this section, we show some $\ell_2$ robust training results. For each dataset, we train three models using standard training, $\ell_\infty$ robust training~\cite{andriushchenko2019provably} with $\ell_\infty$ perturbation norm $\epsilon_\infty$, and $\ell_p$ robust training with $p = 2$ and $\ell_2$ perturbation norm $\epsilon_2$. And in Table \ref{tab:ell_2_stump_training} and \ref{tab:ell_2_tree_training}, we report the verification results of these models from $\ell_2$ verification.

\begin{table*}[htbp]
\begin{center}
\scalebox{0.90}{
\setlength\tabcolsep{1.5pt}
\begin{tabular}{c|c|c?c|c?c|c?c|c}
\hline
    
    \multicolumn{3}{c?}{Dataset}&\multicolumn{2}{c?}{standard training ~}&\multicolumn{2}{c?}{$\ell_{\infty}$ training}&\multicolumn{2}{c}{$\ell_2$ training}\\
    name&$\epsilon_\infty$&$\epsilon_2$&standard err.& $\ell_2$ verified err. &standard err.& $\ell_2$ verified err. &standard err.& $\ell_2$ verified err. \\\hline
    
    breast-cancer & 0.3 &0.7&0.73\%&97.08\%&4.37\%&99.27\%&8.76\%&\textbf{39.42\%}\\\hline
    {Fashion-MNIST shoes}& 0.2 & 0.4 &5.05\%&69.85\%&9.25\%&81.05\%&14.55\%&\textbf{49.55\%}\\\hline
    {MNIST 1 vs. 5} & 0.3 & 0.8 &0.59\%&67.09\%&1.33\%&66.45\%&4.44\%&\textbf{36.56\%}\\\hline
    
    {MNIST 2 vs. 6} & 0.3 & 0.8 &2.81\%&97.74\%&3.91\%&85.52\%&13.67\%&\textbf{76.98\%}\\
    \hline

\end{tabular}
}
\caption{
\textbf{
$\ell_2$ robust training for ensemble stumps}
In this table, we train the model with $p = 2$ and compare the results with $\ell_\infty$ trained models. For each dataset, we train three models using standard training, $\ell_\infty$ norm robust training with $\epsilon_\infty$ and $\ell_2$ norm robust training with $\epsilon_2$. And we test and compare the $\ell_2$ robustness of these models using $\ell_2$ robust verification.
} 
\label{tab:ell_2_stump_training}
\end{center}
\end{table*}

\begin{table*}[htbp]
\begin{center}
\scalebox{0.85}{
\setlength\tabcolsep{1.5pt}
\begin{tabular}{c|c|c|c|c?c|c?c|c?c|c}
\hline
    
    \multicolumn{5}{c?}{Dataset}&\multicolumn{2}{c?}{standard training ~}&\multicolumn{2}{c?}{$\ell_{\infty}$ training~\citep{andriushchenko2019provably}}&\multicolumn{2}{c}{$\ell_2$ training (ours)}\\
    name&$\epsilon_\infty$&$\epsilon_2$&n. trees & depth & standard err.&verified err. &standard err.& verified err. & standard err. & verified err. \\\hline
    
    {Fashion-MNIST shoes} 
    &0.2&0.4&3&5&8.05\%&99.40\%&7.65\%&93.49\%&17.36\%&\bf 68.23\%\\\hline
    {breast-cancer} 
    & 0.3&0.8 & 5 & 5 &  1.47\% & 97.06\% & 1.47\% & 97.79\% & 12.50\% &\bf  55.88\% \\\hline
    {MNIST 1 vs. 5} 
    &0.3&0.8&3&5&2.37\%&100.0\%&2.12\%&97.72\%&23.25\%&\bf 50.54\%\\\hline
    
    {MNIST 2 vs. 6} 
    &0.3 & 0.8&3&5 &3.82\%&100.0\%&3.12\%&100.0\%&19.80\%&\bf 93.56\%\\
    \hline

\end{tabular}
}
\caption{
\textbf{$\ell_2$ robust training for tree ensembles.} We report standard and $\ell_2$ robust test error for all the three methods. We also report $\epsilon_\infty$ and $\epsilon_2$  for each dataset, and the number of trees in each ensemble. 
} 
\label{tab:ell_2_tree_training}
\vspace{-5pt}
\end{center}
\end{table*}

\end{document}